\documentclass[conference]{IEEEtran}
\IEEEoverridecommandlockouts
\usepackage{amsmath,amssymb,amsfonts}
\usepackage{algorithmic}
\usepackage{graphicx}
\usepackage{textcomp}
\usepackage{biblatex}
\usepackage{xcolor}

\usepackage{subfigure}
\usepackage{array}
\usepackage{enumitem}
\usepackage{hyperref}
\usepackage{color}

\usepackage{multirow}

\usepackage{amsmath,amssymb}
\usepackage{multicol}
\usepackage{xcolor}
\usepackage{url}
\usepackage{hyperref}
\usepackage{float}
\usepackage{algorithmic}
\usepackage{algorithm}
\usepackage{soul}
\usepackage{comment}
\usepackage{graphics}
\usepackage[table,xcdraw]{xcolor}
\newcommand{\bX}{\mathbf{X}}

\newcommand{\bS}{\mathbf{S}}
\newcommand{\bSigma}{\mathbf{\Sigma}}

\newcommand{\bW}{\mathbf{W}}
\newcommand{\bY}{\mathbf{Y}}
\newcommand{\bB}{\mathbf{B}}
\newcommand{\bU}{\mathbf{U}}
\newcommand{\bV}{\mathbf{V}}
\newcommand{\bQ}{\mathbf{Q}}
\newcommand{\bH}{\mathbf{H}}

\newcommand{\bm}{\mathbf{m}}

\newcommand{\bI}{\mathbf{I}}
\newcommand{\bM}{\mathbf{M}}

\newcommand{\bs}{\mathbf{s}}

\newcommand{\bx}{\mathbf{x}}

\newcommand{\bZ}{\mathbf{Z}}

\newcommand{\calF}{\cal{F}}
\newcommand{\tz}{\tilde{z}}

\newcommand{\bTheta}{\boldsymbol{\Theta}}

\NewDocumentCommand{\rot}{O{60} O{1em} m}{\makebox[#2][l]{\rotatebox{#1}{#3}}}

\DeclareMathOperator*{\argmax}{arg\,max}

\newtheorem{proposition}{Proposition}
\newtheorem{proof}{Proof}

\def\BibTeX{{\rm B\kern-.05em{\sc i\kern-.025em b}\kern-.08em
    T\kern-.1667em\lower.7ex\hbox{E}\kern-.125emX}}

\begin{document}

\title{Classification EM-PCA for clustering and embedding}



\author{Zineddine Tighidet\IEEEauthorrefmark{3}\IEEEauthorrefmark{5}\IEEEauthorrefmark{1}\thanks{\textbf{*} Work done during time at Centre Borelli, Université Paris Cité.}, 
Lazhar Labiod\IEEEauthorrefmark{4}, 
Mohamed Nadif\IEEEauthorrefmark{4}
    \IEEEauthorblockA{\\
        \IEEEauthorrefmark{3}Sorbonne Université, CNRS, ISIR, Paris, France  \IEEEauthorrefmark{5}BNP Paribas, Paris, France \\
        \IEEEauthorrefmark{4}Centre Borelli UMR 9010, Université Paris Cité, France \\
        \normalsize{
            \textit{firstname.lastname@}\{\textit{sorbonne-universite.fr, u-paris.fr}\}
        }
    }
}

\maketitle

\begin{abstract}
The mixture model is undoubtedly one of the greatest contributions to clustering. For continuous data, Gaussian models are often used and the Expectation-Maximization (EM) algorithm is particularly suitable for estimating parameters from which clustering is inferred. If these models are particularly popular in various domains including image clustering, they however suffer from the dimensionality and also from the slowness of convergence of the EM algorithm. However,  the Classification EM (CEM) algorithm, a classifying version, offers a fast convergence solution while dimensionality reduction still remains a challenge. Thus we propose in this paper an algorithm combining simultaneously and non-sequentially the two tasks --Data embedding and Clustering-- relying on \emph{Principal Component Analysis} (PCA) and CEM. We demonstrate the interest of such approach in terms of clustering and data embedding. We also establish different connections with other clustering approaches.\footnote{This work has been accepted for publication in IEEE Conference on Big Data et the Special Session on Machine Learning. The final published version will be available via IEEE Xplore.}
\end{abstract}

\section{Introduction}
\IEEEPARstart{N}{owadays}, many real-world data sets are high-dimensional. In order to reduce the dimensionality a manifold learning technique can be used to map a set of high-dimensional data into a low-dimensional space, while preserving the intrinsic structure of the data. Principal component analysis ({\tt PCA}) \cite{pca} is the most popular linear approach.
In nonlinear cases, however, there are other more efficient approaches to be found. A number of techniques have been proposed, including Multi-Dimensional Scaling (MDS) \cite{MDS}, Isometric Feature Mapping (ISOMAP) \cite{ISOMAP}, Locally Linear Embedding (LLE) \cite{LLE}, Locally Preserving Projections (LPP) \cite{LPP}, and Stochastic Neighbor Embedding (SNE) \cite{SNE}; we refer the interested reader to \cite{Engel_12}. Unfortunately, these nonlinear techniques tend to be extremely sensitive to noise, sample size, the choice of neighborhood, and other parameters; see for instance \cite{Gittins_12,Van_09,Engel_12}.
In the context of deep learning, where data embedding is referred to as data representation,  an auto-encoder can learn a representation (or {\it encoding}) for a set of data. If linear activations are used, or if there is only a single sigmoid hidden layer, then the optimal solution for an auto-encoder is strongly related to Principal Component Analysis ({\tt PCA}) \cite{hinton2006reducing}.

In data science, data embedding (DE) is commonly used for the purposes of visualizing, but it can also play a significant role in clustering, where the aim is to divide a dataset into homogeneous clusters. Working with a low-dimensional space can be useful when partitioning data, and a number of approaches are reported in the literature with applications in various fields.

\begin{figure}[!h]\begin{center}\includegraphics[scale=0.42]{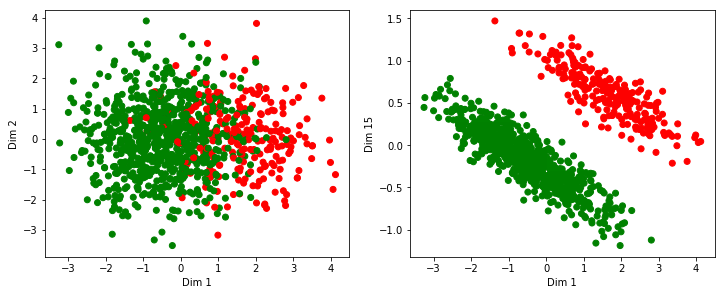} 
\caption{{\tt PCA} projection of the {\tt Chang} dataset onto the plan spawned by the first and second principal components (left) and the plan spawned by the first and fifteenth principal components (right).} 
\label{fig:chang_2d_viz}
\end{center}
\vspace{-0.3cm}
\end{figure}
A popular method to achieve this is the use of principal component analysis ({\tt PCA}), which reduces the dimensionality of data while retaining the most relevant information, followed by any clustering algorithm. 
Although such an approach has been successfully applied sequentially in many applications \cite{pca-gmm}, they present some drawbacks. This is because the first components do not necessarily capture the clustering structure. Chang \cite{chang1983using} discussed a simulated example from a 15-dimensional mixture model into 2 clusters to show the failure of principal components as a method for reducing the dimension of the data before clustering. With the scheme described in \cite{chang1983using} the first 8 variables can be considered roughly as a block of variables with the same correlations, while the rest of the variables form another block.  
Thereby, from the simulated continuous data of size $1000 \times 15$ into 2 clusters we observe that the plan spawned by the two-first components, obtained by {\tt PCA}, does not reveal the two classes Figure~(\ref{fig:chang_2d_viz}, left). However, the plan spawned by the last and the first components perfectly discerns the two classes as shown in Figure~(\ref{fig:chang_2d_viz}, right). Therefore the idea of taking into account only the first components to perform a clustering is not always effective.


In our proposal, we have chosen to rely on the mixture model approach for its flexibility. Thereby, we present an approach that simultaneously uses the {\tt PCA} and the Expectation-Maximization-Type algorithm ({\tt EM}) \cite{dempster1977maximum,mclachlan2007algorithm} that inserts a classification step, referred to as Classification {\tt EM} ({\tt CEM} \cite{celeux1995gaussian}). The derived algorithm called {\tt CEM-PCA} we propose can be viewed as a regularized dimension reduction method.
This regulation will be beneficial to the reduction of the dimension while taking into account the clustering structure to be discovered. In this way, it simultaneously combines the two data tasks --data embedding and clustering--. As reported in Figure \ref{fig:PCA-kmeans-CEM-15-PCs}, we observe the interest of such approach on simulated data. The sub-figure on the left illustrates the obtained clustering by {\tt K-means} \cite{macqueen1967classification} applied on the 15 principal components arising from {\tt PCA}. The sub-figure on the right represents the clusters obtained by {\tt CEM-PCA} that generates a data embedding $\bB$. {\tt CEM-PCA} successfully separates the two classes and results in perfect accuracy (Accuracy=100\%), as opposed to {\tt PCA} followed by {\tt K-means} (Accuracy=78\%). 
%
\begin{figure}[!h]
\centering
\begin{tabular}{cc}
\includegraphics[scale=0.3]{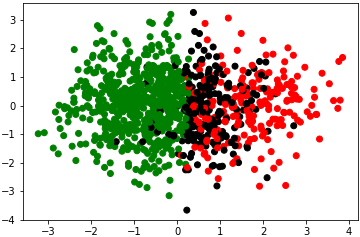} &
  \includegraphics[scale=0.3]{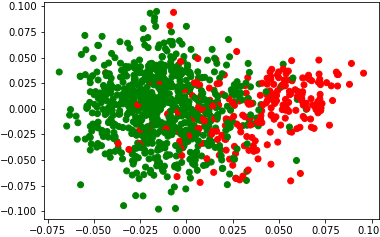} \\
\end{tabular}
\caption{Comparison between the clustering of {\tt K-means} (left) and {\tt CEM-PCA} (right) on Chang data using respectively the components arising from {\tt PCA} for {\tt K-means} and the data embedding $\bB$ obtained by {\tt CEM-PCA}. Black points represent 22\% misclassified objects.}\label{fig:PCA-kmeans-CEM-15-PCs}
\end{figure}

The paper is structured as follows: Section~\ref{sec:related_work_motivation} describes related work and methods to which we compare our approach, Section~\ref{sec:proposal} presents {\tt CEM-PCA} and all the algorithms that are used such as {\tt EM} and {\tt CEM} algorithms. Section~\ref{sec:optimization_algorithm} covers the optimization aspect of {\tt CEM-PCA} as well as the algorithm itself and the complexity analysis. In Section~\ref{sec:experiments} we present the experimental evaluation of our approach, including the results and comparisons with other state-of-the-art methods. In Section~\ref{sec:relation_soa}, we establish connections between {\tt CEM-PCA} and other state-of-the-art methods and finally, we conclude and discuss the potential applications and future directions of the proposed approach.

\section{Related work and Motivation} \label{sec:related_work_motivation}
In this section, we describe several existing methods that are related to our proposed {\tt CEM-PCA} approach:

\begin{itemize}
    \item {\tt Reduced K-means} \cite{clusPCA,Yamamoto_14} combines {\tt PCA} for dimension reduction with {\tt K-means} for clustering.
    \item {\tt EM-GMM} \cite{banfield1993model,Reynolds2009GaussianMM} estimates the parameters of a Gaussian mixture models. The clusters are inferred at the convergence using the maximum a posterior principle. Often wrongly referenced by GMM when it is the {\tt EM} algorithm \cite{dempster1977maximum} estimating the parameters of a Gaussian mixture model.
    \item {\tt PCA-GMM} \cite{pca-gmm} proposes a reduction of the dimensionality of the data in each component of the model by principal component analysis. To learn the low-dimensional parameters of the mixture model they use {\tt EM} algorithm whose M-step requires the solution of constrained optimization problems. 
    \item {\tt Deep-K-means} \cite{fard2020deep} is a jointly clustering with {\tt K-Means} and learning representations by considering the {\tt K-Means} clustering loss as the limit of a differentiable function. 
    \item {\tt Deep-GMM} \cite{viroli2017deep} provides a generalization of classical Gaussian mixtures to multiple layers. Each layer contains a set of latent variables that follow a mixture of Gaussian distributions. To avoid over parameterized solutions, dimension reduction is applied at each layer using factorial models.
\end{itemize}

\section{Our Proposal} \label{sec:proposal}
Before proposing an optimization of an objective function taking into account the simultaneity of the two tasks --data embedding and clustering--, we briefly review the Gaussian mixture model and the EM and CEM algorithms traditionally used for such approach.

\subsection{Gaussian Mixture Model}
The mixture model is undoubtedly one of the greatest contributions to clustering \cite{mclachlan2019finite}. With a finite Gaussian Mixture Model ({\tt GMM}) \cite{Reynolds2009GaussianMM}, the data $\bX=(\bx_1, \ldots, \bx_n)$ are considered as constituting a sample of $n$ independent instances of a random variable $\cal{X}$ in $\mathbb{R}^d$. The probability density function can be expressed as follows:
$$f(\bx_i;\bTheta)=\sum_{k=1}^g \pi_k \varphi_k(\bx_i|\mu_k,\bSigma_k),  \forall i \in \{1,\ldots,n\}$$
where,\\

$\bTheta=(\pi_1,\ldots,\pi_g,\mu_1,\ldots,\mu_g,\bSigma_1,\ldots,\bSigma_g)$. The $\pi_k$'s are the weights or mixing probabilities (such as $\pi_k>0$, $\sum_{k=1}^g \pi_k=1$) and $g$ is the number of components in the mixture.
Thus, the classes are ellipsoidal, centered on the mean vector $\mu_k$, and with other geometric characteristics, such as volume, shape and orientation, determined by the spectral decomposition of covariance matrix $\bSigma_k$ \cite{banfield1993model, celeux1995gaussian}. 
To estimate $\bTheta$, we rely on the log-likelihood maximization given by:
$$
L(\bX;\Theta)
=\sum_{i=1}^n \log \left( \sum_{k=1}^g \pi_k \varphi_k(\bx_i|\mu_k,\bSigma_k)\right).
$$
%
Maximization is usually performed by expectation maximization ({\tt EM}) \cite{dempster1977maximum}; an iterative algorithm based on maximizing the conditional expectation of the log-likelihood of the complete data given $\Theta'$:
$$Q(\Theta|\Theta')=\sum_i\sum_k \tz_{ik}\log(\pi_k \varphi_k(\bx_i|\mu_k,\bSigma_k))$$

where $\tz_{ik}\propto \pi_k \varphi_\ell(\bx_i|\mu_k,\bSigma_k)$ are the posterior probabilities. More specifically, the algorithm is decomposed into two steps (steps \textbf{E}-\textbf{M}) and the unknown parameters of $\Theta$ are updated thanks to the probabilities computed previously. For each component $k$, we have:
\label{parameters}
\[
\pi_k=\frac{\sum_i \tz_{ik}}{n}, \quad \mu_k=\frac{\sum_i \tz_{ik}\bx_i}{\sum_i \tz_{ik}},
\]
\[
\mbox{ and } \bSigma_k=\frac{\sum_i \tz_{ik}(\bx_i-\mu_k)^\top(\bx_i-\mu_k)}{\sum_i s_{ik}}
\]
At convergence, we can deduce a hard clustering using the principle of the \emph{maximum a posterior} by 
\begin{equation}\label{formule:Z}
    z_{ik} \leftarrow \argmax_k \tz_{ik}, \forall i \in \{1,\ldots,n\}.
\end{equation}

\begin{algorithm}
    \caption{{\tt EM-GMM}}\label{alg:gmm}
    \label{alg:gmm_clust}
    \begin{algorithmic}
        \STATE {\bfseries Input:} $\bX \leftarrow (\bx_1,\dots, \bx_n)$, $g$ the number of components;\\
        \STATE {\bfseries Initialization:} initialize $\Theta^\prime$  from a partition obtained with $K$-means;\\
        \REPEAT
            \STATE {\bf E-step}: compute $Q(\Theta|\Theta')$; \\
            \STATE {\bf M-step}: update $\pi_k$, $\mu_k$ and $\bSigma_k$ according to (\ref{parameters});\\
        \UNTIL convergence
        \RETURN $\bSigma_k$'s, $\mu_{k}$'s, $\pi_{k}$'s and $\bZ$ using \ref{formule:Z}
        
    \end{algorithmic}
\end{algorithm}
\subsection{Classification {\tt EM} ({\tt CEM}) Algorithm}
In order to use the model in a clustering context, we want to jointly infer $\bZ$ and learn the model parameters. To this end, we rely on the Classification Expectation-Maximization algorithm \cite{celux1992}, referred to as {\tt CEM}. It consists of inserting a C-step between E-step and M-step. This leads for maximizing
the complete-data log-likelihood given by $$L(\bX,\bZ;\Theta)=\sum_{i=1}^n\sum_{k=1}^g z_{ik} \log \pi_k\varphi_{k}(\bx_{i}|\mu_k,\bSigma_k).$$ With {\tt CEM} this maximization requires fewer iterations and is a simplified version of {\tt EM} capable of handling large data sets.
\begin{algorithm}
    \caption{{\tt CEM}}\label{alg:gmm}
    \label{alg:gmm_clust}
    \begin{algorithmic}
        \STATE {\bfseries Input:} $\bX \leftarrow (\bx_1,\dots, \bx_n)$, $g$ the number of components;\\
        \STATE {\bfseries Initialization:} initialize $\Theta^\prime$  from a partition obtained with $K$-means;\\
        \REPEAT
            \STATE {\bf E-step}: compute $Q(\Theta|\Theta')$; \\
            \STATE {\bf C-step}: $z_{ik} \leftarrow \argmax_k \tz_{ik}, \forall i \in \{1,\ldots,n\}$; \\
            \STATE {\bf M-step}: update $\pi_k$, $\mu_k$ and $\bSigma_k$ according to (\ref{parameters}) where $\tz_{ik} \leftarrow z_{ik}$;\\
        \UNTIL convergence
        \RETURN $\bZ$, $\bSigma_k$'s, $\mu_{k}$'s and  $\pi_{k}$'s.
    \end{algorithmic}
\end{algorithm}

\subsection{Objective function}
The data are represented by a matrix $\bX=(x_{ij})$ of size $n \times d$ where the $\bx_{i}=(x_{i1},\ldots,x_{id})^\top$ are assumed to be sampled from a given parametric distribution of density $\varphi$. The value of each entry in the data matrix depends on the latent row of model parameters. The partition of the set of rows into $g$ classes is represented by the latent classification matrix $\bZ=(z_{ik})$, with $\sum_{k=1}^g z_{ik}=1$, where $z_{ik} = 1$ if row $i$ belongs to the $k$\textit{th} class of rows and $z_{ik}=0$ otherwise. We also write $z_i \in \{1,\ldots, g\}$ as the index of the class of $i$.
We consider the problem of dimension reduction and clustering simultaneously. The idea is to combine the {\tt PCA} and the {\tt CEM} via a regularization. In this way, we propose the following objective function to be minimized:
\begin{align}\label{objective_func}
{\calF}(\bX;\bQ, \bZ,\bSigma,\bS)&=|| \bX-\bB\bQ^\top||^2 +\delta ||\bB-\bM||^2 \nonumber \\&- \sum_{i=1}^n\sum_{k=1}^g z_{ik}\log(\pi_k \varphi(\bm_{i},(\bs_{k},\bSigma_{k})) 
\end{align}
with:
\begin{itemize}
    \item $\bX \in \mathbb{R}^{n \times d}$
    \item $\bQ=\bX^\top\bB \in \mathbb{R}^{d \times p}$
    \item $\bB \in \mathbb{R}^{n \times p}$ is an orthonormal matrix in columns
    \item $\bs_k=(s_{k1},\ldots,s_{kp})$ the centroid of the $k$th class. Then the centroid matrix $\bS \in \mathbb{R}^{g \times p} $
    \item $\bZ \in \{0,1\}^{n \times g}$ clustering matrix
    \item $\varphi(\bm_{i},(\bs_{k},\bSigma_{k}))$ is a distribution matrix in $\bM \in \mathbb{R}^{n \times p}$
    \item The $\pi_k$ the apriori probabilities of the mixture
\end{itemize}

The diagram in Figure~\ref{fig:schema} summarizes in a simplified way the proposed method. The first term of the objective function corresponds to PCA for data embedding while the third term concerns CEM clustering task on the PCA embedding. Note that both tasks – PCA embedding and CEM clustering – are performed simultaneously and supported by the second term ${\left \| \bB - \bM \right \|}^{2} $; this is the key to achieve good embedding while taking the clustering structure into account.

\begin{figure*}
\begin{center} 
\includegraphics[height=5cm, width=17.5cm]{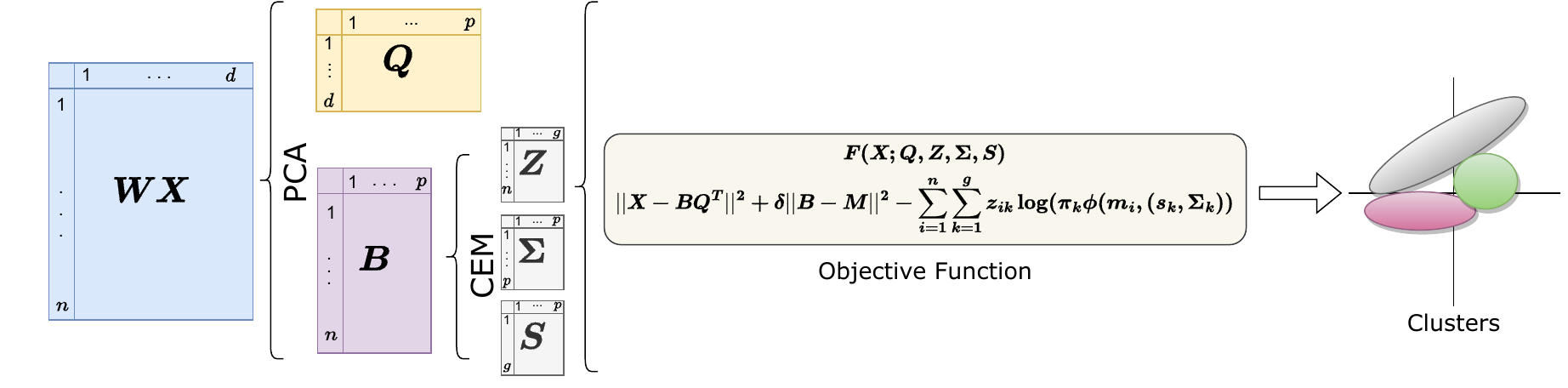} 
\caption{Diagram illustrating the steps of the proposed algorithm ({\tt CEM-PCA}). }
\label{fig:schema}
\end{center}
\end{figure*}

\subsection{Graph Laplacian Regularization}
\label{regul_WX}

Although {\tt PCA} is effective at providing an embedding for data on a linear manifold, it may not be suitable for non-linear manifold applications. To address this issue, the graph Laplacian-based embedding method is often utilized. To implement this method, we begin by constructing a \textit{k}-nearest neighbor data graph using the $n$ data samples $\{\bx_{1},\ldots,\bx_{n}\}$ as vertices. The data weight matrix $\bW$ is defined as follows,

\[
\bW_{ij} = \begin{cases}
            e^{-||\bx_{i} - \bx_{j}||^2}, & \mbox{if } \bx_{j} \in \mathcal{N}(\bx_{i}); i \neq j \\ 
            0 & \mbox{otherwise},
        \end{cases}
\]

where $\mathcal{N}(\bx_{i})$ represents the set of \textit{k}-nearest neighbors of
$\bx_{i}$. Laplacian embedding \cite{neurips_laplacian_emb}, \cite{siam_laplacian_emb} preserves the local geometrical relationships and maximizes the smoothness with respect to the intrinsic manifold of the data set in the low embedding space. We also normalize the resulting matrix $\bW$ by row.

In the subsequent sections, we consider that the matrix $\bX$ is the outcome of the matrix multiplication between the input data and the matrix $\bW$, denoted as $\bX \leftarrow \bW\bX$. We can repeat the operation $m$ times in order to accentuate the smoothing with respect to the manifold of $\bX$ in the low dimensional space: $\bX \leftarrow \bW^m\bX$.

\section{Optimization and Algorithm} \label{sec:optimization_algorithm}
The estimation of the parameters is performed in an alternative way, allowing a mutual reinforcement between the dimensionality reduction task and the clustering task performed on a reduced size matrix.

\subsection{ Optimization}
\label{Optimization}

To solve (\ref{objective_func}), we use an alternated iterative method.

\subsubsection{Computation of $\bQ$ }
First, fixing $\bB$, by setting the derivative of the first term in (\ref{objective_func}) with respect to $Q$ as $0$, we obtain:
\begin{equation}
 \bQ =  \bX^\top \bB
\label{eq:Q}
\end{equation}

\subsubsection{Computation of $\bB$}
Second, given $\bQ$ and $\bM$, we can rewrite (\ref{objective_func}) as:
\begin{equation}\label{eq:B}
\min_{\bB^\top \bB = \bI}  {\left \| \bX-\bB \bQ^\top \right \|}^{2} + \delta {\left \| \bB - \bM \right \|}^{2} . 
\end{equation}
To solve (\ref{eq:B}) we rely on the following proposition.\\

\begin{proposition} \label{TH1}
Let $\bX_{n \times d}$ and $\bQ_{d \times k}$  and $\bM_{n \times k}$ be three matrices. Consider the constrained optimization problem
\begin{eqnarray}
\bB_{*} &=&\arg\min_{\bB} \left \| \bX - \bB \bQ^\top \right \|^{2} +  \delta {\left \| \bB - \bM \right \|}^{2} \mbox{ s.t } \bB^\top \bB = \bI \label{eq:TH00}\nonumber \\
 &=&\arg\max_{\bB} \ \mbox{Tr}((\bX \bQ^+\delta \bM)\bB^\top) \quad \mbox{s.t} \quad \bB^\top \bB = \bI
\label{eq:TH1}
\end{eqnarray}

The solution of Eq. (\ref{eq:TH1}) comes from the singular value decomposition ({\tt SVD}) of  $(\bX \bQ^+\delta \bM)$. Let $U D V^\top$ be the {\tt SVD} for $(\bX \bQ^+\delta \bM)$, then $\bB_{*} = UV^\top$.\\

\end{proposition}
\begin{proof}
We expand the matrix norm $$\left \| \bX - \bB \bQ^\top \right \|^{2} +  \delta {\left \| \bB - \bM \right \|}^{2}.$$
We have
\begin{align*}
&Tr(\bX^\top \bX) - 2 Tr(\bX^\top \bB \bQ^\top)
+ Tr(\bQ \bB^\top \bB \bQ^\top) \nonumber \\
&+ \delta \Big( Tr(\bB^\top \bB)+Tr(\bM^\top \bM)
-2Tr(\bB^\top \bM)\Big).
\label{eq:TH11}
\end{align*}
Since $\bB^\top \bB = \bI$, the original minimization problem (\ref{eq:TH00}) is equivalent to maximize the  term
\begin{equation}
\max_{\bB^\top \bB = \bI} Tr((\bX \bQ+\delta \bM) \bB^\top).
\label{eq:TH2}
\end{equation}
With the {\tt SVD} of $(\bX \bQ+\delta \bM)$ = $UDV^\top$, this middle term becomes
\begin{eqnarray}
Tr((\bX \bQ+\delta \bM) \bB^\top) &=& Tr(U D V^\top \bB^\top) \nonumber \\
                    &=& Tr(U D\hat{\bB}^\top) \quad \mbox{where} \quad \hat{\bB}=\bB V \nonumber \\
                    &=&Tr(\hat{\bB}^\top U D). \label{th22}
\end{eqnarray}
Denoting 
\begin{itemize}
    \item[-] $U=[\mathbf{u}_1|\ldots|\mathbf{u}_k]\in \mathbb{R}^{d \times k}$,
    \item[-] $D=Diag(d_1,\ldots,d_k) \in \mathbb{R}_{+}^{k \times k}$,
    \item[-] $\hat{\bB}=[\mathbf{\hat{b}}_1|\ldots|\mathbf{\hat{b}}_k] \in \mathbb{R}^{d \times k}$.
\end{itemize}
Applying the Cauchy-Shwartz inequality and since $U^\top U=\bI$, $\hat{\bB}^\top \hat{\bB} = \bI$ due to $VV^\top=\bI$, we have 
$$Tr(\hat{\bB}^\top U D) \leq \sum_{i}d_i ||\mathbf{u}_i||\times||\mathbf{\hat{b}}_i||=\sum_i d_i= Tr(D).$$ 
Then the upper bound is clearly attained by setting $\hat{\bB}= U$. This leads to $\hat{\bB}=\bB V= U$ and $\bB VV^\top= UV^\top$. Hence we obtain $ \bB_* = U V^\top.$ 

\end{proof}

\subsubsection{Computation of $\bM$ }
Fixing $\bQ, \bZ,\bSigma,\bS$ and $\bQ$, the objective function with respect to $\bM$ is 

\begin{align}\label{ojective_M}
\delta ||\bB-\bM||^2 - \sum_{i=1}^n\sum_{k=1}^g z_{ik}\log(\pi_k \varphi(\bm_{i},(\bs_{k},\bSigma_{k})) 
\end{align}

We can now derive the updated equation of $\bm_i$ in closed form as follow:

\begin{equation}\label{eq.MM}
\bm_i=\Big(\sum_{k=1}^g z_{ik}\bSigma_{k}^{-1}+\delta \bI\Big)^{-1}\Big(\delta b_i+\sum_{k=1}^g z_{ik}\bSigma_{k}^{-1}\bs_{k}\Big).
\end{equation}

\subsubsection{Computation of $\bs_k,\bSigma_k $ and $\pi_k$ }
The updates of these parameters are the same as those described in the {\tt CEM} algorithm.

\subsection{ Algorithm}

In summary, the steps of the proposed {\tt CEM-PCA} algorithm
can be deduced in Algorithm\ref{alg:model_based_co_clustering:LBCEM}. The code of the {\tt CEM-PCA} algorithm is available on GitHub for reproducibility~\footnote{\href{https://github.com/Zineddine-Tighidet/CEM-PCA}{https://github.com/Zineddine-Tighidet/CEM-PCA}}

\begin{algorithm}
  \caption{{\tt CEM-PCA}}
  \label{alg:model_based_co_clustering:LBCEM}
  \begin{algorithmic}
    \STATE {\bfseries Input:} $\bX$, $g$ number of classes, $p$ the number of latent dimensions, $\delta$ regularization parameter, $k$ the number of neighbors to consider in the Laplacian graph (as described in section~\ref{regul_WX}).
    \STATE {\bfseries Initialization:} Compute $\bB$ and  $\bQ$ using {\tt PCA}, $\bZ$, $\bSigma_{k}$'s and $\bS$, with a {\tt CEM} applied on $\bB$
    \REPEAT
    \STATE {\bfseries Step 1.} Compute $\bM$ using \ref{eq.MM}.
          \STATE {\bfseries Step 2.} Update  $\bZ, \bS, \bSigma_k$'s and the $\pi_k$'s using {\tt CEM}.
          \STATE {\bfseries Step 3.} Compute $\bB=\bU\bV^\top$ where $$\bU \Delta \bV^\top=\text{svd}(\bX\bQ+\delta \bM),$$
        \STATE {\bfseries Step 4.} Compute $\bQ=\bX^\top\bB$ 
    \UNTIL convergence
    \RETURN $\bZ$, $\bB$, $\bM$, $\bSigma_k$'s and  $\pi_k$'s.
  \end{algorithmic}
\end{algorithm}


\subsection{Complexity Analysis}

In this section we analyze the complexity of each step for each iteration of the {\tt CEM-PCA} algorithm.

\textbf{Step 1.} and \textbf{Step 2.} estimate $\bQ$, $\bZ$, $\bM$, $\bS$, and $\bSigma$ which can be done using the {\tt EM} algorithm based on a {\tt GMM}. We can use a more efficient version of {\tt EM} algorithm for {\tt GMM} called Incremental Gaussian Mixture Network ({\tt IGMN}) \cite{DBLP:journals/corr/PintoE15} that reduces the complexity to $O(nkp^2)$.

\textbf{Step 3.} and \textbf{Step 4.} estimate $\bB$ and $\bQ$ which can be done using a {\tt PCA} that is based on Singular Value Decomposition ({\tt SVD}), the complexity is $O(d^2n+d3)$. In fact, the complexity of covariance matrix computation is $O(d^2n)$. Its eigenvalue decomposition is $O(d^3)$ which gives $O(d^2n+d^3)$. Thereby, the complexity per iteration for the {\tt CEM-PCA} algorithm is $O(nkp^2 + d^2n + d^3)$. Let $t$ be the number of iterations necessary for convergence, the complexity of {\tt CEM-PCA} is $O(t(nkp^2 + d^2n + d^3))$. The use of GPUs can significantly speed up the computation time, due to the multiple matrix operations (e.g. covariance matrix computation).

\section{Experiments} \label{sec:experiments}
\subsection{Datasets}
To evaluate the performance of the {\tt CEM-PCA} algorithm, experiments are carried out on different types of labeled data sets described in Table \ref{tab:datasets}.

\paragraph{Synthetic Datasets} {\tt Atom}, {\tt Chainlink}, {\tt Hepta}, {\tt Lsun3D}, and {\tt Tetra}.
\paragraph{Image Datasets} {\tt COIL20}, {\tt USPS}, {\tt Yale}, and {\tt ORL}.
\paragraph{Biomedical Datasets} {\tt Lung}, {\tt Yeast}, and {\tt Breast}
\paragraph{Text Datasets}
{\tt BBC}, {\tt Classic3}, and {\tt Classic4}. 
%
%
We used the cased {\tt GloVe} model \cite{pennington-etal-2014-glove}, with a vocabulary size of 2.2M and 840 billion tokens represented in 300 dimensions. To obtain the embedding of each document we computed the average of the word embeddings that compose each document. The {\tt GloVe} model can be described as a log-bilinear approach that employs a weighted least-squares objective. The fundamental idea behind the model is based on the observation that word-word co-occurrence probabilities ratios can be utilized to capture meaning.
\begin{table}[h!]
\vspace{-0.5cm}
\caption{{Characteristics of synthetic data, image, biomedical and textual datasets }\label{tab:datasets}}
\scriptsize
    \begin{center}
        \begin{tabular}{lccccc}
        \hline
        Set & Data & Type & \#rows & \#columns & \#classes\\
        \hline
        \multirow{5}{*}{\textbf{FCPS}} & {\tt Atom} & Numerical & 800 & 3 & 2 \\
        & {\tt Chainlink} & Numerical & 1000 & 3 & 2 \\
        & {\tt Hepta} & Numerical & 212 & 3 &  7 \\
        & {\tt Lsun3D} & Numerical & 404 & 3 & 4 \\
        & {\tt Tetra} & Numerical & 400 & 3 & 4 \\
        \hline \hline
        \multirow{4}{*}{\textbf{Images}} & {\tt COIL20} & Numerical & 1440 & 1440 & 20 \\
        & {\tt USPS} & Numerical & 9298 & 256 & 10 \\
        & {\tt Yale} & Numerical & 165 & 1024 &  15 \\
        & {\tt ORL} & Numerical & 400 & 1024 & 40 \\
        \hline \hline
        \multirow{3}{*}{\textbf{Biomedical}} & {\tt Lung} & Microarray & 203 & 12600 & 5 \\
        & {\tt Yeast} & Microarray & 1484 & 8 & 10 \\
        & {\tt Breast} & Microarray & 106 & 9 &  6 \\
        \hline \hline
        \multirow{3}{*}{\textbf{Text}} & {\tt BBC} & Text & 2225 & 300 & 5 \\
        & {\tt Classic3} & Text & 3891 & 300 & 3 \\
        & {\tt Classic4} & Text & 7095 & 300 &  4 \\
        \hline
        \end{tabular}
    \end{center}
    \vspace{-0.5cm}
\end{table}
%

\begin{figure*}[!t]
\centering
\includegraphics[scale=0.37]{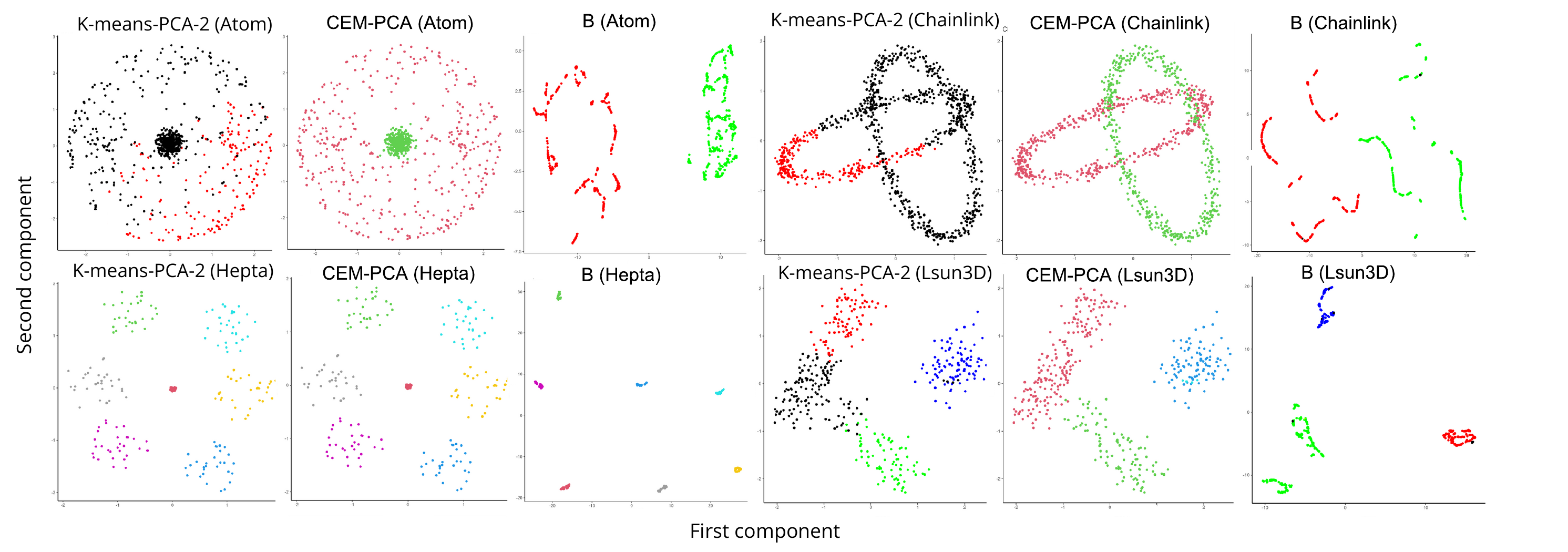} 
\vspace{-0.5cm}
\caption{Comparison between the clustering of {\tt K-means-PCA-2} and {\tt CEM-PCA}, and the representation of the data embedding $\bB$ -- the methods where applied on {\tt FCPS} datasets and the plots where obtained using UMAP (black points represent misclassified objects). {\tt K-means-PCA-2} results from applying {\tt K-means} on the two first components of {\tt PCA}, {\tt CEM-PCA} is our proposed method, and $\bB$ is the data embedding obtained by {\tt CEM-PCA} (see algorithm \ref{alg:model_based_co_clustering:LBCEM}.) This figure highlights the advantage of applying dimension reduction and clustering simultaneously ({\tt CEM-PCA}) rather than sequentially ({\tt K-means-PCA-2}). The embedding space of {\tt CEM-PCA} (represented by $\bB$ in the figure) perfectly captures the separation of the clusters, making it easy for the {\tt CEM} algorithm to perform clustering.}
\label{fig:FCPS}
\end{figure*}
\subsection{Performance Metrics}
To measure the clustering performance of the proposed method, we used the commonly adopted metrics, normalized mutual information ({\tt NMI}) \cite{nmi}, Adjusted Rand Index ({\tt ARI}) \cite{Steinley2004PropertiesOT}, and {\tt Accuracy}. 

The clustering accuracy noted ({\tt Acc}) discovers the one-to-one relationship between two partitions and measures the extent to which each cluster contains data points from the corresponding class. It is defined by
$\frac{1}{n} \sum_{i = 1}^{n}\delta(C_{i}, map(P_{i}))
$
where $n$ is the total number of samples, $P_{i}$ is the $i^{th}$ obtained
cluster and $C_{i}$ is the true $i^{th}$ class provided by the data
set. $\delta(x, y)$ is the delta function that equals one if $x = y$
and equals zero otherwise, and $map(P_{i})$ is the permutation
mapping function that maps the obtained label $P_{i}$ to the
equivalent label from the data set. The best mapping can be
found by using the Kuhn-Munkres algorithm \cite{lovász2009matching}. However, the accuracy measure is not reliable when the classes are dramatically unbalanced. In this case, we prefer other measures such as the {\tt NMI} and {\tt ARI} metrics. The {\tt NMI} is estimated by:
\begin{align*}
\text{NMI} = \frac{\sum_{k,\ell} \frac{n_{k\ell}}{n}\log \frac{nn_{k\ell}}{n_{k}\hat{n}_{\ell}}}{\sqrt{(\sum_{k} n_{k}\log \frac{n_{k}}{n})(\sum_{\ell} \hat{n}_{\ell} \log \frac{\hat{n}_{\ell}}{n})}}
\end{align*}
where $n_{k}$ represents the number of data contained in the class $C_{k}(1 \leq k \leq g)$, $\hat{n}_{\ell}$ is the number of data belonging to the class $P_{\ell}(1 \leq \ell \leq g)$, and $n_{k\ell}$ represents the number of data that are at the intersection between the class $C_{k}$ and the class $P_{\ell}$. The {\tt ARI} metric is a measure of the similarity between two partitions. From a mathematical point of view, the {\tt ARI} is related to the precision. The adjusted form of the Rand Index is as follows:
\begin{align*}
\text{ARI}= \frac{\sum_{k,\ell} \binom{n_{k\ell}}{2} - \left[ \sum_{k} \binom{n_{k}}{2} \sum_{\ell} \binom{\hat{n}_{\ell}}{2} \right] \big/ \binom{n}{2} }{\frac{1}{2} \left[ \sum_{k} \binom{n_{k}}{2} + \sum_{\ell} \binom{\hat{n}_{\ell}}{2} \right] - \left[ \sum_{k} \binom{n_{k}}{2} \sum_{\ell} \binom{\hat{n}_{\ell}}{2} \right] \big/ \binom{n}{2}}
\end{align*}

For all the measures ({\tt Acc}, {\tt NMI}, and {\tt ARI}), a value close to 1 means a good clustering.

\subsection{Parameters Setting}


\subsubsection{Number of iterations and initializations}
The {\tt CEM-PCA} algorithm is initialized randomly. Based on 20 initializations, the best trial that minimizes (\ref{objective_func}) is retained. For convergence, the algorithm requires less than 20 iterations.

\subsubsection{Number of Dimensions $p$}
We trained our model on different numbers of dimensions to consider from the {\tt PCA}. 
In our experiments, $p=10$ appears to be a good choice.




\subsubsection{Regularization parameter $\delta$}
The parameter $\delta$ has a significant impact on the obtained results. As indicated in the objective function~(\ref{objective_func}), the larger $\delta$ is, the more importance is given to classification (i.e. minimizing the intra-class inertia). On the contrary, if we take a smaller $\delta$ we will give more importance to the dimension reduction (i.e. minimize the distance between the input data and the principal components). We performed {\tt CEM-PCA} on different values of $\delta$ ranging from $10^{-6}$ to 1. In our experiments a chosen $\delta$ between $10^{-6}$ and $10^{-5}$ seems to give the best performance overall. In other words, a slight regularization is enough to obtain good results.


\subsection{Comparisons and evaluations}

\subsubsection{Clustering}
We used a Gaussian kernel to build the Laplacian matrix $\bW$ which resulted in perfect results. This enhances even more the smoothness with respect to the intrinsic manifold of the data set in the low embedding space while preserving the local geometrical relationships. The results obtained on synthetic data ({\tt FCPS}) are reported in Table~\ref{tab:clust_perf} and illustrated in Figure~\ref{fig:FCPS}. In terms of Acc, NMI and ARI, we notice that the performance of the proposed method ({\tt CEM-PCA}) is significantly better than the other reference methods, {\tt K-means}, {\tt K-means-PCA-2} which denotes {\tt K-means} applied on the two first components of {\tt PCA} while {\tt K-means-PCA-10} is applied on the 10 first components when $d>10$, {\tt Reduced-Kmeans}, {\tt Deep-\textit{K}-means}, {\tt EM-GMM}, {\tt PCA-GMM}, {\tt Deep-GMM}. 

We first observe that performing the two steps in a sequential manner can lead to unexpected clustering as illustrated by {\tt K-means-PCA-2} and {\tt K-means-PCA-10}. In contrast, {\tt CEM-PCA} that performs both steps simultaneously, creates a mutual reinforcement between dimension reduction and clustering. This performance is confirmed on Images, Biomedical, and Textual datasets (Table~\ref{tab:clust_perf}).
Finally, we notice that {\tt deep-GMM} and {\tt Deep-Kmeans} with a high-cost computing and the {\tt PCA-GMM} method which proceeds by reduction of the dimensionality by cluster appear much less effective than {\tt CEM-PCA} while {\tt Reduced K-means} and {\tt EM-GMM} can be sometimes competitive.

\begin{table*}[!t]
\caption{{Clustering Performance ({\tt NMI}, {\tt ARI}, and Acc -- Accuracy) on the {\tt FCPS}, {\tt Images}, {\tt BioMedical}, and {\tt Text} data sets. The best results are in bold and the second best results are underlined and "-"  corresponds to the case where $d<10$.}}
\label{tab:clust_perf}
\centering
\begin{tabular}{lllllllllll}
\cline{2-11}
 Dataset & Metric & \rot{\tt K-means}  & \rot{\tt Kmeans-PCA-2} & \rot{\tt Kmeans-PCA-10} & \rot{\tt Reduced-Kmeans} & \rot{\tt Deep-k-means} & \rot{\tt EM-GMM} & \rot{\tt PCA-GMM} & \rot{\tt Deep-GMM} & \rot{\tt CEM-PCA} \\
 \hline
 Atom &
 \begin{tabular}{@{}c@{}c@{}}NMI \\ ARI \\ Acc\end{tabular} &
 \begin{tabular}{@{}c@{}c@{}}0.29 \\ 0.18 \\ 0.71\end{tabular} & 
  \begin{tabular}{@{}c@{}c@{}}0.29 \\ 0.18 \\ 0.71\end{tabular} & 
 \begin{tabular}{@{}c@{}c@{}}\textbf{--} \\ \textbf{--} \\ \textbf{--}\end{tabular} & 
 \begin{tabular}{@{}c@{}c@{}}0.56 \\ 0.17 \\ 0.58\end{tabular} & 
 \begin{tabular}{@{}c@{}c@{}}0.33 \\ 0.23 \\ 0.40\end{tabular} & 
 \begin{tabular}{@{}c@{}c@{}}\underline{0.95} \\ \underline{0.98} \\ \underline{0.91}\end{tabular} & 
 \begin{tabular}{@{}c@{}c@{}}0.77 \\ 0.74 \\ 0.85\end{tabular} &
 \begin{tabular}{@{}c@{}c@{}}0.23 \\ 0.15 \\ 0.37\end{tabular} &
 \begin{tabular}{@{}c@{}c@{}}\textbf{1.0} \\ \textbf{1.0} \\ \textbf{1.0}\end{tabular} 
 \\ \hline
 Chainlink & 
 \begin{tabular}{@{}c@{}c@{}}NMI \\ ARI \\ Acc\end{tabular} &
 \begin{tabular}{@{}c@{}c@{}}0.06 \\ 0.10 \\ 0.50\end{tabular} &  
  \begin{tabular}{@{}c@{}c@{}}0.01 \\ 0.01 \\ 0.52\end{tabular} & 
 \begin{tabular}{@{}c@{}c@{}}\textbf{--} \\ \textbf{--} \\ \textbf{--}\end{tabular} & 
 \begin{tabular}{@{}c@{}c@{}}0.06 \\ 0.10 \\ 0.09\end{tabular} &  
 \begin{tabular}{@{}c@{}c@{}}0.07 \\ 0.10 \\ 0.12\end{tabular} &  
 \begin{tabular}{@{}c@{}c@{}}\underline{0.84} \\ \underline{0.91} \\ \underline{0.95} \end{tabular} & 
 \begin{tabular}{@{}c@{}c@{}}0.31 \\ 0.28 \\ 0.41\end{tabular} & 
 \begin{tabular}{@{}c@{}c@{}}0.34 \\ 0.40 \\ 0.55 \end{tabular} & 
 \begin{tabular}{@{}c@{}c@{}}\textbf{0.96} \\ \textbf{0.98} \\ \textbf{0.99}\end{tabular} 
 \\ \hline
 Hepta &
 \begin{tabular}{@{}c@{}c@{}}NMI \\ ARI \\ Acc\end{tabular} &
 \begin{tabular}{@{}c@{}}\textbf{1.0} \\ \textbf{1.0} \\ \textbf{1.0}\end{tabular} &  
   \begin{tabular}{@{}c@{}c@{}}1.0 \\ 1.0 \\ 1.0\end{tabular} & 
 \begin{tabular}{@{}c@{}c@{}}\textbf{--} \\ \textbf{--} \\ \textbf{--}\end{tabular} & 
 \begin{tabular}{@{}c@{}c@{}}\textbf{1.0} \\ \textbf{1.0} \\ \textbf{1.0}\end{tabular} &
 \begin{tabular}{@{}c@{}c@{}}0.90 \\ 0.88 \\ 0.92\end{tabular} &
 \begin{tabular}{@{}c@{}c@{}}\textbf{1.0} \\ \textbf{1.0} \\ \textbf{1.0}\end{tabular} & 
 \begin{tabular}{@{}c@{}c@{}}0.90 \\ 0.78 \\ 0.82\end{tabular} &
 \begin{tabular}{@{}c@{}c@{}}\underline{0.92} \\ \underline{0.90} \\ \underline{0.95}\end{tabular} & 
 \begin{tabular}{@{}c@{}c@{}}\textbf{1.0} \\ \textbf{1.0} \\ \textbf{1.0} \end{tabular}
 \\ \hline
 Lsun3D & 
 \begin{tabular}{@{}c@{}c@{}}NMI \\ ARI \\ Acc\end{tabular} &
 \begin{tabular}{@{}c@{}c@{}}0.73 \\ 0.60 \\ 0.73\end{tabular} &
    \begin{tabular}{@{}c@{}c@{}}0.74 \\ 0.61 \\ 0.74\end{tabular} & 
 \begin{tabular}{@{}c@{}c@{}}\textbf{--} \\ \textbf{--} \\ \textbf{--}\end{tabular} & 
 \begin{tabular}{@{}c@{}c@{}}0.65 \\ 0.59 \\ 0.71\end{tabular} &
 \begin{tabular}{@{}c@{}c@{}}0.84 \\ 0.86 \\ 0.89\end{tabular} &
 \begin{tabular}{@{}c@{}c@{}}\textbf{1.0} \\ \textbf{1.0} \\ \underline{0.97}\end{tabular} & 
 \begin{tabular}{@{}c@{}c@{}}0.74 \\ 0.61 \\ 0.71\end{tabular} &
 \begin{tabular}{@{}c@{}c@{}}0.65 \\ 0.60 \\ 0.73 \end{tabular} & 
 \begin{tabular}{@{}c@{}c@{}}\underline{0.98} \\ \underline{0.99} \\ \textbf{0.98}\end{tabular}
 \\ \hline
 Tetra & 
 \begin{tabular}{@{}c@{}c@{}}NMI \\ ARI \\ Acc\end{tabular} &
 \begin{tabular}{@{}c@{}c@{}}\textbf{1.0} \\ \textbf{1.0} \\ \textbf{1.0}\end{tabular} &
     \begin{tabular}{@{}c@{}c@{}}0.75 \\ 0.73 \\ \underline{0.88}\end{tabular} & 
 \begin{tabular}{@{}c@{}c@{}}\textbf{--} \\ \textbf{--} \\ \textbf{--}\end{tabular} & 
 \begin{tabular}{@{}c@{}c@{}}\textbf{1.0} \\ \textbf{1.0} \\ \textbf{1.0}\end{tabular} &
 \begin{tabular}{@{}c@{}c@{}}0.72 \\ 0.69 \\ 0.79\end{tabular} &
 \begin{tabular}{@{}c@{}c@{}}\textbf{1.0} \\ \textbf{1.0} \\ 1.0\end{tabular} & 
 \begin{tabular}{@{}c@{}c@{}}0.64 \\ 0.54 \\ 0.61\end{tabular} &
 \begin{tabular}{@{}c@{}c@{}}\underline{0.72} \\ \underline{0.73} \\ 0.84\end{tabular} & 
 \begin{tabular}{@{}c@{}c@{}}\textbf{1.0} \\ \textbf{1.0} \\ \textbf{1.0}\end{tabular}
 \\ \hline \hline
 COIL20 & 
 \begin{tabular}{@{}c@{}c@{}}NMI \\ ARI \\ Acc\end{tabular} &
 \begin{tabular}{@{}c@{}c@{}}0.72 \\ 0.55 \\ 0.60\end{tabular} &
     \begin{tabular}{@{}c@{}c@{}}0.63 \\ 0.40 \\ 0.52\end{tabular} & 
 \begin{tabular}{@{}c@{}c@{}}\textbf{0.77} \\ 0.57 \\ 0.65\end{tabular} & 
 \begin{tabular}{@{}c@{}c@{}}0.73 \\ \textbf{0.63} \\ \underline{0.68} \end{tabular} &
 \begin{tabular}{@{}c@{}c@{}}0.49 \\ 0.27 \\ 0.51\end{tabular} &
  \begin{tabular}{@{}c@{}c@{}}0.60 \\ 0.34 \\ 0.61\end{tabular} & 
 \begin{tabular}{@{}c@{}c@{}}0.36 \\ 0.13 \\ 0.23 \end{tabular} &
 \begin{tabular}{@{}c@{}c@{}}0.62 \\ 0.33 \\ 0.67\end{tabular} & 
 \begin{tabular}{@{}c@{}c@{}}\underline{0.75} \\ \underline{0.58} \\ \textbf{0.82} \end{tabular}
 \\ \hline
 USPS & 
 \begin{tabular}{@{}c@{}c@{}}NMI \\ ARI \\ Acc\end{tabular} &
 \begin{tabular}{@{}c@{}c@{}}\underline{0.57} \\ \underline{0.47} \\ \textbf{0.58}\end{tabular} &
     \begin{tabular}{@{}c@{}c@{}}0.37 \\ 0.28 \\ 0.39\end{tabular} & 
 \begin{tabular}{@{}c@{}c@{}}0.54 \\ 0.45 \\ \textbf{0.58}\end{tabular} & 
 \begin{tabular}{@{}c@{}c@{}} 0.54 \\ \underline{0.47} \\ 0.50\end{tabular} &
 \begin{tabular}{@{}c@{}c@{}}0.40 \\ 0.32 \\ \underline{0.51}\end{tabular} &
 \begin{tabular}{@{}c@{}c@{}}0.31 \\ 0.15 \\ 0.47\end{tabular} & 
 \begin{tabular}{@{}c@{}c@{}}0.39 \\ 0.28 \\ 0.43\end{tabular} &
 \begin{tabular}{@{}c@{}c@{}}0.42 \\ 0.29 \\ 0.49\end{tabular} & 
 \begin{tabular}{@{}c@{}c@{}}\textbf{0.62} \\ \textbf{0.50} \\ \textbf{0.58} \end{tabular}
 \\ \hline
 Yale & 
 \begin{tabular}{@{}c@{}c@{}}NMI \\ ARI \\ Acc\end{tabular} &
 \begin{tabular}{@{}c@{}c@{}}0.60 \\ \underline{0.52} \\ 0.55\end{tabular} &
     \begin{tabular}{@{}c@{}c@{}}0.45 \\ 0.18 \\ 0.37\end{tabular} & 
 \begin{tabular}{@{}c@{}c@{}}0.64 \\ 0.41 \\ 0.53\end{tabular} & 
 \begin{tabular}{@{}c@{}c@{}}0.52 \\ 0.26 \\ 0.35\end{tabular} &
 \begin{tabular}{@{}c@{}c@{}}0.48 \\ 0.19 \\ 0.53\end{tabular} &
 \begin{tabular}{@{}c@{}c@{}}\underline{0.65} \\ 0.43 \\ \underline{0.70}\end{tabular} & 
 \begin{tabular}{@{}c@{}c@{}}0.12 \\ 0.01 \\ 0.33\end{tabular} &
 \begin{tabular}{@{}c@{}c@{}}0.42 \\ 0.29 \\ 0.55 \end{tabular} & 
 \begin{tabular}{@{}c@{}c@{}}\textbf{0.66} \\ \textbf{0.53} \\ \textbf{0.71} \end{tabular}
 \\ \hline
 ORL & 
 \begin{tabular}{@{}c@{}c@{}}NMI \\ ARI \\ Acc\end{tabular} &
 \begin{tabular}{@{}c@{}c@{}}0.71 \\ 0.31 \\ 0.57\end{tabular} &
     \begin{tabular}{@{}c@{}c@{}}0.57 \\ 0.16 \\ 0.58\end{tabular} & 
 \begin{tabular}{@{}c@{}c@{}}0.75 \\ 0.40 \\ 0.53\end{tabular} & 
 \begin{tabular}{@{}c@{}c@{}} 0.73 \\ 0.39 \\ 0.75 \end{tabular} &
 \begin{tabular}{@{}c@{}c@{}}0.61 \\ 0.20 
 \\ 0.42 \end{tabular} &
 \begin{tabular}{@{}c@{}c@{}}\textbf{0.82} \\ \textbf{0.54} \\ 
\underline{0.79} \end{tabular} & 
 \begin{tabular}{@{}c@{}c@{}}0.23 \\ 0.21 \\ 0.32 \end{tabular} &
  \begin{tabular}{@{}c@{}c@{}}0.55 \\ 0.13 \\ 0.47 \end{tabular} &
 \begin{tabular}{@{}c@{}c@{}}\underline{0.75} \\ \underline{0.42} \\ \textbf{0.82} \end{tabular}
 \\ \hline \hline
 Lung  &
 \begin{tabular}{@{}c@{}c@{}}NMI \\ ARI \\ Acc\end{tabular} &
 \begin{tabular}{@{}c@{}c@{}}0.57 \\ 0.38 \\ 0.43\end{tabular} &
     \begin{tabular}{@{}c@{}c@{}}0.28 \\ 0.14 \\ 0.49\end{tabular} & 
 \begin{tabular}{@{}c@{}c@{}}0.31 \\ 0.15 \\ 0.50\end{tabular} & 
 \begin{tabular}{@{}c@{}c@{}}\underline{0.63} \\ \underline{0.46} \\ 0.75\end{tabular} &
 \begin{tabular}{@{}c@{}c@{}}0.18 \\ 0.03 \\ \underline{0.54}\end{tabular} &
 \begin{tabular}{@{}c@{}c@{}}0.23 \\ 0.12 \\ 0.22\end{tabular} & 
 \begin{tabular}{@{}c@{}c@{}}0.45 \\ 0.39 \\ 0.51 \end{tabular} & 
 \begin{tabular}{@{}c@{}c@{}}0.43 \\ 0.18 \\ 0.48\end{tabular} & 
 \begin{tabular}{@{}c@{}c@{}}\textbf{0.64} \\ \textbf{0.71} \\ \textbf{0.76} \end{tabular}
 \\ \hline
 Yeast &
 \begin{tabular}{@{}c@{}c@{}}NMI \\ ARI \\ Acc\end{tabular} &
 \begin{tabular}{@{}c@{}c@{}}0.25 \\ 0.12 \\ \underline{0.21}\end{tabular} &
     \begin{tabular}{@{}c@{}c@{}}0.20 \\ 0.09 \\ 0.33\end{tabular} & 
 \begin{tabular}{@{}c@{}c@{}}\textbf{--} \\ \textbf{--} \\ \textbf{--}\end{tabular} & 
 \begin{tabular}{@{}c@{}c@{}}\underline{0.26} \\ \underline{0.13} \\ 0.19 \end{tabular} &
 \begin{tabular}{@{}c@{}c@{}}0.07 \\ 0.02 \\ 
0.10 \end{tabular} &
  \begin{tabular}{@{}c@{}c@{}}0.24 \\ 0.14 \\ \textbf{0.41}\end{tabular} & 
 \begin{tabular}{@{}c@{}c@{}}0.13 \\ 0.08 \\ 0.12\end{tabular} &
 \begin{tabular}{@{}c@{}c@{}}0.11 \\ 0.10 \\ \underline{0.21}\end{tabular} &
 \begin{tabular}{@{}c@{}c@{}}\textbf{0.27} \\ \textbf{0.16} \\ \underline{0.24} \end{tabular}
 \\ \hline
 Breast &
 \begin{tabular}{@{}c@{}c@{}}NMI \\ ARI \\ Acc\end{tabular} &
 \begin{tabular}{@{}c@{}c@{}}0.33 \\ 0.18 \\ 0.33\end{tabular} &
     \begin{tabular}{@{}c@{}c@{}}0.69 \\ 0.80 \\ 0.94\end{tabular} & 
 \begin{tabular}{@{}c@{}c@{}}\textbf{--} \\ \textbf{--} \\ \textbf{--}\end{tabular} & 
 \begin{tabular}{@{}c@{}c@{}}\underline{0.70} \\ \textbf{0.81} \\ \underline{0.95}\end{tabular} &
 \begin{tabular}{@{}c@{}c@{}}0.60 \\ 0.69 \\ 0.72\end{tabular} &
 \begin{tabular}{@{}c@{}c@{}}0.65 \\ 0.72 \\ 0.70\end{tabular} & 
 \begin{tabular}{@{}c@{}c@{}}0.18 \\ 0.13 \\ 0.33\end{tabular} &
 \begin{tabular}{@{}c@{}c@{}}\underline{0.66} \\ \underline{0.78} 
\\ 0.81 \end{tabular} &
 \begin{tabular}{@{}c@{}c@{}}\textbf{0.71} \\ \textbf{0.81} \\ \textbf{0.96} \end{tabular}
 \\ \hline \hline
 BBC &
 \begin{tabular}{@{}c@{}c@{}}NMI \\ ARI \\ Acc\end{tabular} &
 \begin{tabular}{@{}c@{}c@{}}\underline{0.81} \\ \underline{0.82} \\ \underline{0.92} \end{tabular} &
     \begin{tabular}{@{}c@{}c@{}}0.63 \\ 0.61 \\ 0.81\end{tabular} & 
 \begin{tabular}{@{}c@{}c@{}} \underline{0.81} \\ \textbf{0.83} \\ \underline{0.92}\end{tabular} & 
 \begin{tabular}{@{}c@{}c@{}} 0.78 \\ 0.77 \\ 0.80\end{tabular} &
 \begin{tabular}{@{}c@{}c@{}}0.52 \\ 0.44 \\ 0.51\end{tabular} &
 \begin{tabular}{@{}c@{}c@{}}0.52 \\ 0.44 \\ 0.58 \end{tabular} & 
 \begin{tabular}{@{}c@{}c@{}}0.60 \\ 0.59 \\ 0.62\end{tabular} &
 \begin{tabular}{@{}c@{}c@{}} 0.78 \\ 0.78 
\\ 0.84 \end{tabular} & 
 \begin{tabular}{@{}c@{}c@{}}\textbf{0.83} \\ 0.81 \\ \textbf{0.94} \end{tabular}
 \\ \hline
 Classic3 &
 \begin{tabular}{@{}c@{}c@{}}NMI \\ ARI \\ Acc\end{tabular} &
 \begin{tabular}{@{}c@{}c@{}}0.90 \\ 0.93 \\ 0.90\end{tabular} &
     \begin{tabular}{@{}c@{}c@{}}0.89 \\ 0.92 \\ \underline{0.97}\end{tabular} & 
 \begin{tabular}{@{}c@{}c@{}}0.89 \\ 0.92 \\ \underline{0.97}\end{tabular} & 
 \begin{tabular}{@{}c@{}c@{}}0.90 \\ 0.93 \\ \underline{0.97} \end{tabular} &
 \begin{tabular}{@{}c@{}c@{}}0.89 \\ 0.92 \\ 0.88\end{tabular} &
 \begin{tabular}{@{}c@{}c@{}}\underline{0.91} \\ \underline{0.94} \\ 0.93\end{tabular} & 
 \begin{tabular}{@{}c@{}c@{}}0.54 \\ 0.50 \\ 0.48\end{tabular} &
 \begin{tabular}{@{}c@{}c@{}}0.84 \\ 0.88 \\ 0.90\end{tabular} & 
 \begin{tabular}{@{}c@{}c@{}}\textbf{0.96} \\ \textbf{0.98} \\ \textbf{0.98} \end{tabular}
 \\ \hline
 Classic4 &
 \begin{tabular}{@{}c@{}c@{}}NMI \\ ARI \\ Acc\end{tabular} &
 \begin{tabular}{@{}c@{}c@{}}\underline{0.66} \\ \underline{0.46} \\ 0.70\end{tabular} &
     \begin{tabular}{@{}c@{}c@{}}0.53 \\ 0.33 \\ 0.52\end{tabular} & 
 \begin{tabular}{@{}c@{}c@{}}0.66 \\ 0.46 \\ 0.70\end{tabular} & 
 \begin{tabular}{@{}c@{}c@{}}\underline{0.66} \\ \underline{0.46} \\ 0.70\end{tabular} &
 \begin{tabular}{@{}c@{}c@{}}0.64 \\ 0.44 \\ \underline{0.71}\end{tabular} &
 \begin{tabular}{@{}c@{}c@{}}0.60 \\ 0.41 \\ 0.69\end{tabular} & 
 \begin{tabular}{@{}c@{}c@{}}0.51 \\ 0.43 \\ 0.69\end{tabular} &
 \begin{tabular}{@{}c@{}c@{}}0.47 \\ 0.29 \\ 0.62 \end{tabular} & 
 \begin{tabular}{@{}c@{}c@{}}\textbf{0.71} \\ \textbf{0.54} \\ \textbf{0.74} \end{tabular}
 \\ \hline
\end{tabular}
\vspace{-0.1cm}
\end{table*}
The results obtained with different values of {\tt NMI} are summarized by Critical Difference (CD) diagrams in Figure \ref{fig:nemenyi_images}.  The aim of CD diagrams \cite{demvsar2006statistical} is to visualize the performance ranks of each approach over the different datasets. If we take the example of $\rho= 10\%$, the CD diagram summarizes the scores given in Table \ref{tab:clust_perf}. It depicts the average rank of each method and the bold line corresponds to the critical difference, based on the post-hoc Nemenyi test \cite{nemenyi1963distribution}. We observe the interest of {\tt CEM-PCA} which clearly outeperforms  other methods including deep versions relying on {\tt GMM} or not.
\begin{figure}[!h]
\centering
\includegraphics[scale=0.22]{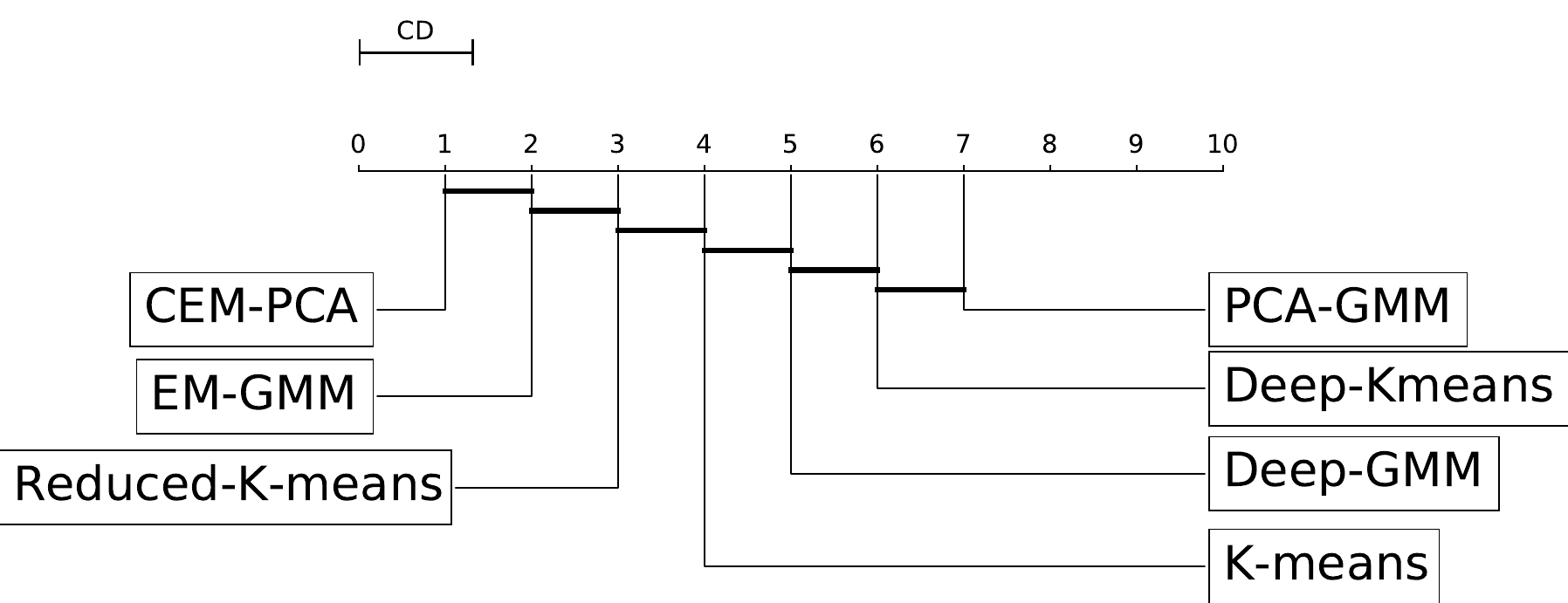} 
\caption{Comparison of all models against each other with the Nemenyi test. Groups of models that are not
significantly different (at $\rho = 0.10$) are connected (based on NMI).}
\label{fig:nemenyi_images}
\end{figure}
\subsubsection{Data embedding}
To go further and measure the quality of data embedding expressed by $\bB$, we carried out projections using the nonlinear {\tt UMAP} method \cite{mcinnes2018umap}. We observe that the separability is sufficiently expressed by this matrix which benefits from the regularization of the third clustering term (\ref{objective_func}) as illustrated in Figures \ref{fig:FCPS}, \ref{fig:umap_images}, and \ref{fig:umap_biomed_text}.
\begin{figure*}[!h]
\includegraphics[scale=0.36]{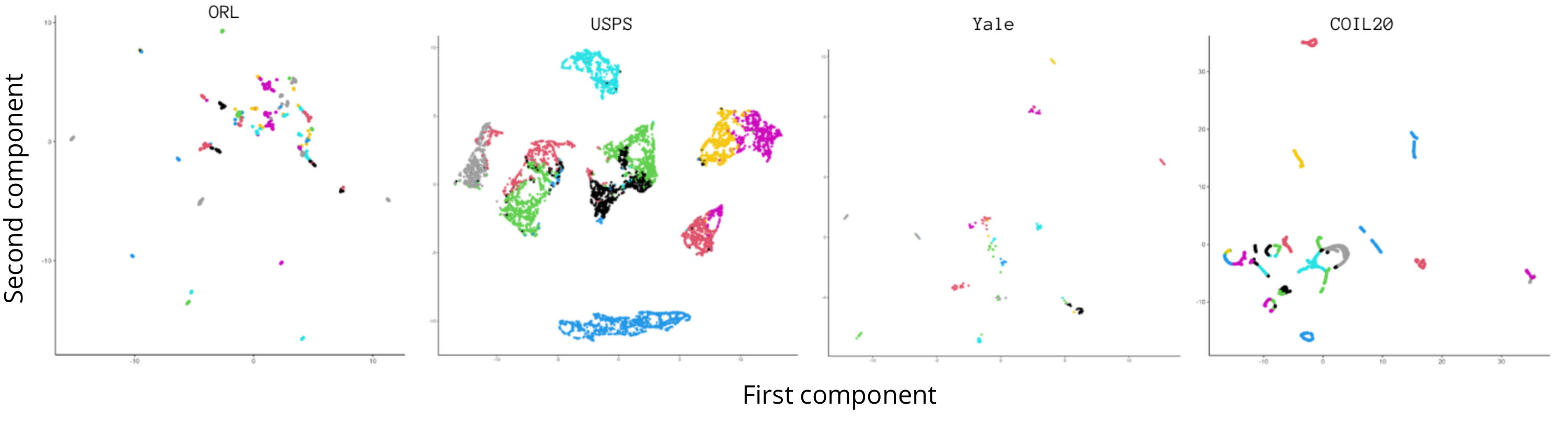} 
\caption{{\tt UMAP} projection and clustering of the {\tt CEM-PCA} method on Images data.}
\label{fig:umap_images}
\end{figure*}
\begin{figure*}[!h]
\includegraphics[scale=0.36]{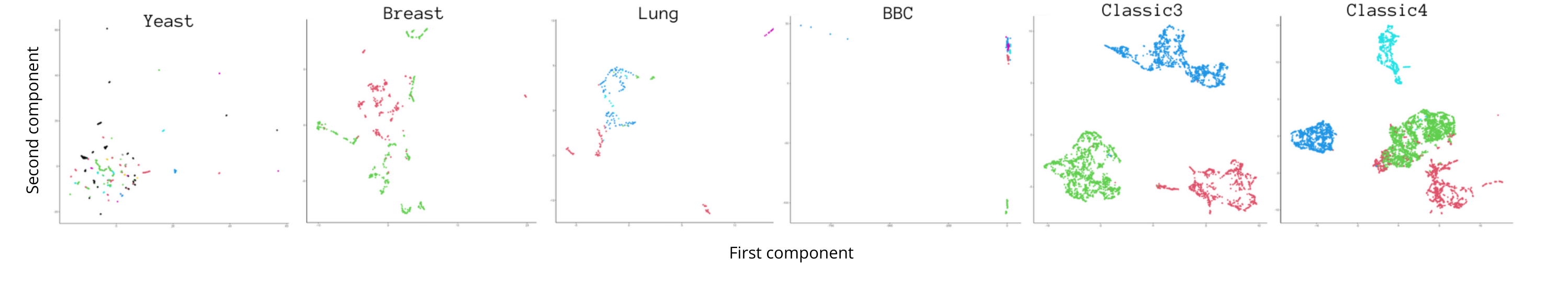}
\caption{{\tt UMAP} projection and clustering of the {\tt CEM-PCA} method on {\tt BioMedical} data (Yeast, Breast, and Lung) and text data (BBC, Classic3, and Classic4).}
\label{fig:umap_biomed_text}
\end{figure*}
%

%




\section{Relationship Between {\tt CEM-PCA} and State-of-the-Art Methods} \label{sec:relation_soa}

Next, we show how our proposed {\tt CEM-PCA} approach is related to some other clustering and data embedding methods.
\subsection{Relation with {\tt EM-GMM} and {\tt K-means}}
The third term of (\ref{objective_func}) corresponds exactly to the minus complete-data log-likelihood from $\bM$:
\begin{align}
- \sum_{i=1}^n\sum_{k=1}^g z_{ik}\log(\pi_k \varphi(\bm_{i},(\bs_{k},\bSigma_{k})).
\end{align}
Its minimization is equivalent to the maximization of the complete-data log-likelihood performed by the {\tt CEM} algorithm. Note that Parsimonious parameterizations \cite{banfield1993model,celeux1995gaussian} of the covariances matrices can be obtained by means of an eigen-decomposition of the form $\bSigma_k=\lambda_k D_k A_k D_k^\top$, where $\lambda_k$ is a scalar controlling the volume of the ellipsoid, $A_k$ is a diagonal matrix specifying the shape of the density contours with $det(A_k) =1$, and $D_k$ is an orthogonal matrix which determines the orientation of the corresponding ellipsoid. Thereby, when $\bSigma_k=\lambda \bI$ and $\pi_k=\pi$ $\forall k$, the {\tt CEM} algorithm is {\tt K-means}. In other words {\tt K-means} is a constrained {\tt CEM} algorithm applied on a reduced data matrix $\bM$ and the k-means objective function takes the following form,
\vspace{-0.5em}
\begin{equation}\label{kmobjective}
||\bM-\bZ\bS||^2.
\end{equation}
\vspace{-2em}
\subsection{\tt Smoothed PCA}
The new data representation is referred to as 
$\tilde{\bX}=\bW\bX$ of size $(n \times d)$ can be viewed as a multiplicative way to encode information from both $\bW$ and $\bX$. 
Then the objective of {\tt CEM-PCA}, defines a graph regularized {\tt PCA} on the smoothed matrix $\tilde{\bX}$. Thus, the first term in (\ref{objective_func}) performs {\tt PCA}, on the centroid computed on the neighborhood (barycenter) of each node. In fact, from 
\vspace{-1em}
$$
 \min_{\mathbf{B},\mathbf{Q}} \,  {\left \| \tilde{\bX}-\mathbf{B}\mathbf{Q}^{\top} \right \|}^2
$$
plugging the optimal solution $\mathbf{Q}=\tilde{\bX}^\top\mathbf{B}$ leads to the equivalent trace maximization problem of a {\tt smoothed PCA} 
\begin{equation}\label{M}
 \max_{\mathbf{B}} Tr (\mathbf{B}^\top (\tilde{\bX}\tilde{\bX}^\top) {\mathbf{B}}).
\end{equation}

\subsection{{\tt Smoothed PCA}  with  clustering regularization}
The objective function of {\tt CEM-PCA} takes the following form
\begin{align}\label{regsmpca}
 \max_{\mathbf{B}} Tr (\mathbf{B}^\top (\tilde{\bX}\tilde{\bX}^\top) {\mathbf{B}} &+  \delta ||\bB-\bM||^2 \nonumber \\
 &- \sum_{i=1}^n\sum_{k=1}^g z_{ik}\log(\pi_k \varphi(\bm_{i},(\bs_{k},\bSigma_{k})) .
\end{align}
This shows that the first term is related to a {\tt Smoothed PCA} and the third term represents a  clustering regularization that enriches the {\tt Smoothed PCA} by plugging the hidden clustering structure via $\mathbf{Z}$. 
%

\subsection{Reduced k-means}
In this subsection we will focus our attention on the objective of Reduced k-means given in \cite{clusPCA,Yamamoto_12}, which can be expressed as follows  
\begin{equation}\label{eq:reducedkm}
 {\left \| \mathbf{X}-\mathbf{Z}\mathbf{S}\mathbf{Q}^\top \right \|}^{2}  = {\left \|\mathbf{X}-\mathbf{X}\mathbf{Q} \mathbf{Q}^\top \right \|}^{2} + {\left \|\mathbf{XQ}-\mathbf{Z}\mathbf{S} \right \|}^{2}. 
\end{equation}
As with {\tt CEM-PCA}, the first term of \ref{eq:reducedkm} performs {\tt PCA} and the second term is the {\tt K-means} objective applied on the reduced {\tt PCA} embedding, which is a special case of the third term in {\tt CEM-PCA} objective. Then, {\tt Reduced K-means} objective is equivalent to {\tt CEM-PCA} objective with the following additional constraints $\pi_k=\pi, \forall k$ and $\bSigma_{k}=\bI, \forall k$.

\subsection{Regularized spectral data embedding}
By setting $\tilde{\bX}=\bW$  we rely only on the initial graph information without integrating any additional information.
From the first term of the objective function reported in (\ref{objective_func})
$
 \min_{\mathbf{B},\mathbf{Q}} \,  {\left \| \bW-\mathbf{B}\mathbf{Q}^{\top} \right \|}^2
$
plugging the optimal solution $\mathbf{Q}=\bW^\top\mathbf{B}$ leads to the following equivalent trace maximization problem of a {\tt Spectral embedding} 
\begin{equation}\label{M}
 \max_{\mathbf{B}} Tr (\mathbf{B}^\top (\bW\bW^\top) {\mathbf{B}}).
\end{equation}
Minimizing the first term in (\ref{objective_func}) is equivalent to the trace maximization problem expressed by
$\max_{\mathbf{B}} \,  Tr(\mathbf{B}^{\top}\mathbf{W}\mathbf{
W}^{\top}\mathbf{B}). $
It is easy to show that $\mathbf{W}$ and
$\mathbf{W}\mathbf{W}^{\top}$ have the same eigenvectors as in classical spectral clustering. Let $\mathbf{B}\Lambda \mathbf{B}^{\top}$ be the eigendecomposition of $\mathbf{W}$. The eigendecomposition of $\mathbf{W}\mathbf{W}^\top=\mathbf{B}\Lambda^2 \mathbf{B}^{\top}$.
Then, the objective of {\tt CEM-PCA} in (\ref{objective_func}), is equivalent  to regularized spectral data embedding by the complete-data log-likelihood,
\begin{align*}\label{regsmpca}
  &\max_{\mathbf{B},\bZ,\bM,\bS,\pi_k,\bSigma_{k}} Tr \big(\mathbf{B}^\top (\bW \bW^\top){\mathbf{B}}\big) \\
  &+ \sum_{i=1}^n\sum_{k=1}^g z_{ik}\log(\pi_k \varphi(\bm_{i},(\bs_{k},\bSigma_{k}))\quad\mbox{s.t}\quad \bM=\bB.
\end{align*}
Assuming that {\tt K-means} as a special case of {\tt CEM}, the objective function of {\tt CEM-PCA} is reduced to 
\begin{equation*}
  \max_{\mathbf{B},\bZ,\bM,\bS} Tr \big(\mathbf{B}^\top (\bW \bW^\top){\mathbf{B}}\big) - ||\bM-\bZ\bS||^2\quad\mbox{s.t}\quad \bM=\bB.
\end{equation*}

\subsection{Graph Convolutional Networks}
Graph Convolutional Networks (GCNs) have experienced significant attention and have become popular methods for learning graph representations. Below, we establish the connection between the graph convolution operator of GCN and the closed-form embedding solution of the {\tt CEM-PCA} formulation. We demonstrate that its provided embedding is {\tt SVD} of the GCN embedding and then can achieve  better or similar results to GCN over several benchmark datasets. 

Similar to other neural networks stacked with repeated layers, GCN contains multiple graph convolution layers; each of which is followed by a nonlinear activation \cite{kipf2016semi,wu2019simplifying,chen2020simple}.
Let $\bH^{(\ell)}$ be the $\ell$-th layer hidden representation, then GCN follows:
\begin{equation}\label{gcn1}
\bH^{(\ell+1)} = \sigma(\bW \bH^{(\ell)} \bQ^{(\ell)}) 
\end{equation}
where $\bQ^{(\ell)}$ is the $\ell$-th layer parameter (to be learned), $\bH^{(0)}=\bX$ and $\sigma$ is the nonlinear activation function. Graph convolution operation is defined as the formulation before activation in (\ref{gcn1}).
 The graph convolution (parameterized with $\bQ$) mapping the feature matrix $\bX$ to a new representation $\bY$  defined as $\bY =\bW\bX\bQ$. 
The embedding solution of {\tt CEM-PCA} is given by the closed form solution of the following problem 
\begin{equation}
\max_{\mathbf{B}} \; Tr\Big((\tilde{\mathbf{X}}\mathbf{Q}+\delta \bM) \mathbf{B}^{\top}\Big) \quad \mbox{s.t.} \quad \mathbf{B}^{\top} \mathbf{B} = \mathbf{I}.
\end{equation}
Let $\hat{\bU}\hat{\Sigma}\hat{\bV}^\top$ be {\tt SVD} of $((\tilde{\mathbf{X}}\mathbf{Q}+\delta \bM) \mathbf{B}^{\top})=(\bW\bX\bQ+\delta \bM)=(\bY+\delta \bM)$, then the embedding $\bB=\hat{\bU}\hat{\bV}^\top$ is obtained from {\tt SVD} of $\bY$.

\section{Conclusion}
In unsupervised learning, {\tt PCA} and {\tt K-means} are popular for their simplicity and effectiveness in dimensionality reduction, partitioning, and visualization, though their sequential use suffers from differing objectives. We propose a method unifying dimension reduction and clustering via Gaussian mixtures \cite{banfield1993model}, overcoming challenges of dimensionality and slow EM convergence through joint optimization of embedding ({\tt PCA}) and clustering ({\tt CEM}, an extension of {\tt K-means}). This scalable approach reveals more meaningful structures than sequential or deep methods, opening perspectives in representation learning and clustering. Beyond traditional applications, our method could also offers promise for interpretability in large language models by uncovering meaningful clusters in embeddings and intermediate transformer layers \cite{tighidet-etal-2024-probing,tighidet-etal-2025-context}.

\printbibliography

\end{document}